\documentclass[12pt]{article}
\usepackage{amsmath,amssymb,amsthm,thmtools,natbib}
\usepackage{graphicx,tikz,subcaption,xcolor}
\usetikzlibrary{shapes,arrows}
\usepackage{algorithm,algorithmic}
\usepackage{url} 

\newcommand{\blind}{0}

\newcommand{\E}{\mathbb{E}}
\newcommand{\p}{\mathbb{P}}
\newcommand{\R}{\mathbb{R}}

\theoremstyle{plain}
\newtheorem{example}{Example}
\newtheorem{definition}{Definition}
\newtheorem{theorem}{Theorem}

\newtheorem{assumption}{Condition}

\begin{document}

\def\spacingset#1{\renewcommand{\baselinestretch}%
{#1}\small\normalsize} \spacingset{1}


\if0\blind
{
  \title{\bf On Learning and Testing of Counterfactual Fairness through Data Preprocessing}
  \author{Haoyu Chen, Wenbin Lu, Rui Song \\
    Department of Statistics, North Carolina State University\\
    and \\
    Pulak Ghosh \\
    Decision Sciences \& Centre for Public Policy, Indian Institute of Management
    }
\date{\empty}
  \maketitle
} \fi

\if1\blind
{
  \bigskip
  \bigskip
  \bigskip
  \begin{center}
    {\LARGE\bf On Learning and Testing of Counterfactual Fairness through Data Preprocessing}
\end{center}
  \medskip
} \fi

\bigskip
\begin{abstract}
	Machine learning has become more important in real-life decision-making but people are concerned about the ethical problems it may bring when used improperly. Recent work brings the discussion of machine learning fairness into the causal framework and elaborates on the concept of Counterfactual Fairness. In this paper, we develop the Fair Learning through dAta Preprocessing (FLAP) algorithm to learn counterfactually fair decisions from biased training data and formalize the conditions where different data preprocessing procedures should be used to guarantee counterfactual fairness. We also show that Counterfactual Fairness is equivalent to the conditional independence of the decisions and the sensitive attributes given the processed non-sensitive attributes, which enables us to detect discrimination in the original decision using the processed data. The performance of our algorithm is illustrated using simulated data and real-world applications.
\end{abstract}
		
\noindent%
\vfill

\newpage
\spacingset{1.5} 
\section{Introduction}\label{sec:introduction}
The rapid popularization of machine learning methods and the growing availability of personal data have enabled decision-makers from various fields such as graduate admission~\citep{waters2014grade}, hiring~\citep{ajunwa2016hiring}, credit scoring~\citep{thomas2009consumer}, and criminal justice~\citep{brennan2009evaluating} to make data-driven decisions efficiently. However, the community and the authorities have also raised concern that these automatically learned decisions may inherit the historical bias and discrimination from the training data and would cause serious ethical problems when used in practice~\citep{nature2016more,angwin2016bias,dwoskin2015social,executive2016big}.

Consider a training dataset $\mathcal{D}$ consisting of sensitive attributes $S$ such as gender and race, non-sensitive attributes $A$ and decisions $Y$. If the historical decisions $Y$ are not fair across the sensitive groups, a powerful machine learning algorithm will capture this pattern of bias and yield learned decisions $\hat{Y}$ that mimic the preference of the historical decision-maker, and it is often the case that the more discriminative an algorithm is, the more discriminatory it might be. 

While researchers agree that methods should be developed to learn fair decisions, opinions vary on the quantitative definition of fairness. In general, researchers use either the \textit{observational} or \textit{counterfactual} approaches to formalize the concept of fairness. The observational approaches often describe fairness with metrics of the observable data and predicted decisions~\citep{hardt2016equality,chouldechova2017fair,yeom2018discriminative}. For example, Demographic Parity (DP) or Group Fairness~\citep{zemel2013learning,khademi2019fairness} considers the learned decision $\hat{Y}$ to be fair if it has the same distribution for different sensitive groups, i.e., $P(\hat{Y}|S = s) = P(\hat{Y}|S = s')$. The Individual Fairness (IF) definition~\citep{dwork2012fairness} views fairness as treating similar individuals similarly, which means the distance between $\hat{Y}(s_i, a_i)$ and $\hat{Y}(s_j, a_j)$ should be small if individuals $i$ and $j$ are similar.

The other branch of fairness and/or discrimination definitions are built upon the causal framework of \citet{pearl2009causal}, such as direct/indirect discrimination~\citep{zhang2017causal,nabi2018fair}, path-specific effect~\citep{wu2019pc}, counterfactual error rate~\citep{zhang2018equality} and counterfactual fairness~\citep{kusner2017counterfactual,wang2019equal,wu2019counterfactual}. These definitions often involve the notion of counterfactuals, which means what the attributes or decision would be if an individual were in a different sensitive group. With the help of the potential outcome concept, the measuring of fairness is no longer restricted to the observable quantities~\citep{kilbertus2017avoiding,zhang2018fairness}. For example, the Equal Opportunity (EO) definition \cite{wang2019equal} has the same idea as IF but it can directly compare the actual and counterfactual decisions of the same individual instead of the actual decisions of two similar individuals. The Counterfactual Fairness (CF) definition~\citep{kusner2017counterfactual} or equivalently, the Affirmative Action (AA) definition~\citep{wang2019equal} goes one step further than EO and derives the counterfactual decisions from the counterfactual non-sensitive attributes. It first asks what the non-sensitive attributes $A$ would be had $S$ been different and then compare the counterfactual decisions of an individual given her/his counterfactual non-sensitive attributes. For example, a female student with a test score of 85 and a male student with the same score should have the same probability of being admitted under the Equal Opportunity definition. When considering the Counterfactual Fairness definition, we first imagine a counterfactual world where the female student were treated as a boy since she was born. There she received the same educational resources as her male siblings in her family and finally reached a score of 95. We would then use 95 as her counterfactual non-sensitive attribute and conclude that she should have a higher probability of being admitted than her male competitor. The fact that her score is 85 in the real world is due to the historical disadvantage of limited education resources and the Equal Opportunity definition will ignore this kind of unfairness.
We adopt CF as our definition of fairness and it is formally described in Section \ref{sec:model}. We believe causal reasoning is the key to fair decisions as \citet{dedeo2014wrong} pointed out that even the most successful algorithms would fail to make fair judgments due to the lack of causal reasoning ability. 

For the observational definitions, fair decisions can be learned by solving optimization problems, either adding the fairness condition as a constraint~\citep{dwork2012fairness} or directly optimize the fairness metric as an objective function~\citep{zemel2013learning}. When using the counterfactual definitions, however, an approximation of the causal model or the counterfactuals is often needed since the counterfactuals are unobservable. In the FairLearning algorithm proposed by \citet{kusner2017counterfactual}, the unobserved parts of the graphical causal model are sampled using the Markov chain Monte Carlo method. Then they use only the non-descendants of $S$ to learn the decision, which ensures CF but will have a low prediction accuracy. In \citet{wang2019equal}, the counterfactual of $A$ had $S$ been $s'$ is imputed as the sum of the counterfactual group mean $\E(A|S=s')$ and the residuals from the original group $A - \E(A|S=s)$. As we discuss later, this approach would only work when a strong assumption of the relationship between $A$ and $S$ is satisfied.

\subsection{Contributions}
We develop the Fair Learning through dAta Preprocessing (FLAP) algorithm to learn counterfactually fair decisions from biased training data. While current literature is vague about the assumptions needed for their algorithms to achieve fairness, we formalize the weak and strong conditions where different data preprocessing procedures should be used to guarantee CF and prove the results under the causal framework of \citet{pearl2009causal}. We show that our algorithm can predict fairer decisions with similar accuracy when compared with other counterfactual fair learning algorithms using three simulated datasets and three real-world applications, including the loan approval data from a fintech company, the adult income data, and the COMPAS recidivism data. 

On the other hand, the processed data also enable us to detect discrimination in the original decision. We prove that CF is equivalent to the conditional independence of the decisions and the sensitive attributes given the processed non-sensitive attributes under certain conditions. Therefore any well-established conditional independence tests can be used to test CF with the processed data. To our knowledge, it is the first time that a formal statistical test for CF is proposed. We illustrate the idea using the Conditional Distance Correlation test~\citep{wang2015conditional} in our simulation and test the fairness of the decisions in the loan approval data using a parametric test.

\section{Causal Model and Counterfactual Fairness}\label{sec:model}
For the discussion below, we consider the sensitive attributes $S \in \mathcal{S}$ to be categorical, which is a reasonable restriction for the commonly discussed sensitive information such as race and gender. The non-sensitive attributes $A \in \mathcal{A} \subseteq \R^d$, and the decision $Y$ is binary as admit or not in graduate admission, hire or not in the hiring process, approve or not in loan assessment.

\begin{figure}[!htbp]
	\centering
	\begin{tikzpicture}[
	roundnode/.style={circle, draw=black, thick, minimum size=6mm},
	ellipsenode/.style={ellipse, draw=black, thick, minimum size=6mm},
	squarednode/.style={rectangle, draw=black, thick, minimum size=6mm},
	]
	\node[squarednode](uyhat) at (0,0) {$U_{\hat{Y}}$};
	\node[roundnode](yhat) at (1.2,0) {$\hat{Y}$};
	\node[roundnode](y) at (2.4,0) {$Y$};
	\node[squarednode](uy) at (3.6,0) {$U_Y$};
	\node[roundnode](s) at (1.2,1.2) {$S$};
	\node[squarednode](us) at (0,1.2) {$U_S$};
	\node[roundnode](a) at (2.4,1.2) {$A$};
	\node[squarednode](ua) at (3.6,1.2) {$U_A$};
	
	\node[draw,align=left] at (8,0.6) {\small$\begin{aligned}
	    S &= f_S(U_S),\\
	    A &= f_A(S, U_A),\\
	    Y &= f_Y(S, A, U_Y),\\
	    \hat{Y} &= f_{\hat{Y}}(S, A, U_{\hat{Y}}) = \mathbf{1}\{U_{\hat{Y}}<p(S, A)\}.
	\end{aligned}$};
	
	\draw[thick, ->] (uyhat.east) -- (yhat.west);
	\draw[thick, ->] (us.east) -- (s.west);
	\draw[thick, ->] (s.east) -- (a.west);
	\draw[thick, ->] (s.south east) -- (y.north west);
	\draw[thick, ->] (s.south) -- (yhat.north);
	\draw[thick, ->] (ua.west) -- (a.east);
	\draw[thick, ->] (a.south) -- (y.north);
	\draw[thick, ->] (a.south west) -- (yhat.45);
	\draw[thick, ->] (uy.west) -- (y.east);
    \end{tikzpicture}
	\caption{Structural causal model.}
	\label{fig:M}
\end{figure}
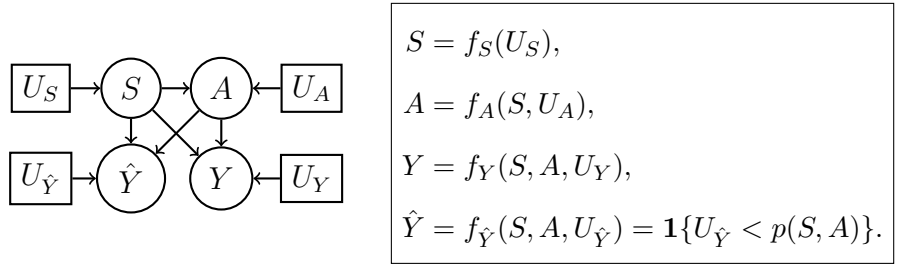

To bring the discussion of fairness into the framework of causal inference, we begin by constructing the Structural Causal Model (SCM) for the data. As described in \cite{pearl2009causality}, an SCM $M$ consists of a set of exogenous variables $U$, a set of endogenous variables $V$, and $F$, a set of functions that assign value to each endogenous variable given its parents in $V$ and the exogenous variables $U$. In our case (Figure \ref{fig:M}), we consider $V=\{S, A, Y, \hat{Y}\}$, where $\{S, A, Y\}$ are the observed data and $\hat{Y}$ is the prediction of $Y$ we made based on $S$ and $A$. The only exogenous variable affecting $\hat{Y}$ is a $\mathrm{Uniform}(0,1)$ random variable $U_{\hat{Y}}$ so that we can conveniently express the value of $\hat{Y}$ with a structural equation. We assume that $U_S$, $U_A$, and $U_Y$, which are the exogenous variables that affect $S$, $A$, and $Y$ respectively, are independent of each other. The structural equations on the right side of Figure \ref{fig:M} are described with the functions in $F$, one for each component in $V$. Here we express $f_{\hat{Y}}$ as an indicator function so that $\hat{Y}$ is a Bernoulli random variable that takes value one with probability $p(S, A)$. In general, $p(s, a)$ could be any function that maps $\mathcal{S}\times\mathcal{A}$ to $[0,1]$, but we are more interested in such functions that will result in a fair decision, more details of which will be discussed in Section \ref{sec:algorithm}. It can be seen that the subset of exogenous variables $\{U_S, U_A, U_Y\}$ characterize everything we should know about a unit. Any two units with the same realization will have the same behavior and result irrespective of the other differences in their identities.
	
Here we give a simplified loan approval model as a running example to help understand the SCM we considered.
	
\begin{example}\label{ex:1}
A bank asks each loan applicant for her/his race $S$ and annual income $A$ to decide to approve the application ($Y=1$) or not ($Y=0$). There are two races in the population of the applicants, $S=1$ represents the advantageous group, and $S=0$ for the disadvantageous one. Let $U_S\sim\mathrm{Uniform}(0,1)$, we generate $S=\mathbf{1}\{U_S < 0.7\}$. The annual income is log-normally distributed for each race group and its scale and location parameters may depend on race: 
$$
A = c_1\exp\{c_2 + \lambda_a S + c_3 \sigma_a^S U_A\},
$$
where $U_A$ is a standard normal random variable, $c_1, c_3>0$, and $c_2$ are constants that affect the median and spread of the population income, $\lambda_a$ decides the difference in mean log income between the two race groups, and $\sigma_a>0$ determines the standard deviation ratio of the log incomes. The decision by the bank can be simulated from a logistic model:
$$
Y=\mathbf{1}\{U_Y < \mathrm{expit}(\beta_0 + \beta_a A + \beta_s S)\},
$$
where $U_Y\sim\mathrm{Uniform}(0,1)$ and $\mathrm{expit}(u) = (1 + e^{-u})^{-1}$.
\end{example}
	
In this example, $\beta_s$ characterizes the direct effect of the sensitive attribute on the decision: when $\beta_s>0$, the applications from the advantageous group are more likely to be approved by the bank when holding the income fixed. On the other hand, $\lambda_a$ partly describes the indirect effect because when both $\lambda_a$ and $\beta_a$ are positive, the advantageous group will have a higher income than the other group on average and thus be favored by the bank even if $\beta_s=0$. It is worth noting that, apart from the difference in the mean, the difference in higher moments could also cause unfairness indirectly as alluded to in \citet{fuster2020predictably}. In general, if there are any differences in the distribution of $A$ across the categories in $\mathcal{S}$, the decision based on $A$ might be unfair. However, the indirect effect caused by the differences in the higher moments of $A$ could be case dependent and thus harder to interpret. In our case, $\sigma_a>1$ will lead to a higher average income and hence higher approval probability on average for the advantageous group since the income distribution is right-skewed.
	
With the SCM in hand, we are ready to define the causal quantity we are interested in. Since most sensitive attributes, such as gender and race, cannot be altered in experiments, we will look into the counterfactuals, namely, what the results $Y$ would be had $S$ been different from the observed facts. This quantity is expressed as $Y_{s}(U)$ had $S$ been $s$ for a random unit with exogenous variables $U$ sampled from the population. Define $M_s$ to be the modified SCM from $M$ (Figure \ref{fig:M}) with the equation for $S$ replaced with $S=s$. Then for any realization $U=u$, the unit level counterfactuals $Y_{s}(u)$ can be calculated from $M_s$. Similarly, we can define $\hat{Y}_{s}(U)$ and $\hat{Y}_{s}(u)$ as the counterfactual predicted decision and its realization. The counterfactual fairness can then be defined on both the decision and the prediction based on the counterfactual result. Here we denote $\mathcal{Y}$ as a placeholder for either $Y$ or $\hat{Y}$.
	
\begin{definition}{Counterfactual Fairness.} Given a new pair of attributes $(s^*, a^*)$, a (predicted) decision $\mathcal{Y}$ is counterfactually fair if for any $s' \in \mathcal{S}$,
\begin{equation*}
	\mathcal{Y}_{s'}(U)|\{S = s^*, A = a^*\} \overset{d}{=} \mathcal{Y}_{s^*}(U)|\{S = s^*, A = a^*\}.
\end{equation*}
\end{definition}
	
In other words, the conditional distribution of the counterfactual result should not depend on the sensitive attributes. It should be noted that there are two stages in evaluating the conditional counterfactuals. The first is updating the conditional distribution of $U$. Take the decision $Y$ from Example \ref{ex:1}, if $s^* = 0$, then $U_S|\{S=s^*, A=a^*\}$ is from $\mathrm{Uniform}(0.7, 1)$ and $U_A|\{S=s^*, A=a^*\}$ is a constant $(\log(a^*/c_1) - c_2)/c_3$, but $U_Y|\{S=s^*, A=a^*\}$ is still a $\mathrm{Uniform}(0, 1)$ random variable since $U_Y$ is independent of $S$ and $A$ from the SCM. The next stage is deriving the conditional distribution of the counterfactuals from the structural equations of $M_s$ and the conditional distribution of $U$. Continuing with our example, $Y_1(U)|\{S=0, A=a^*\}$ would be equal in distribution to
\begin{align*}
&f_Y(1, f_A(1, U_A), U_Y)|\{S=0, A=a^*\}\\
\overset{d}{=}&f_Y(1, f_A(1, (\log(a^*/c_1) - c_2)/c_3), U_Y)\\
\overset{d}{=}&\mathbf{1}\{U_Y < \mathrm{expit}(\beta_0 + \beta_a c_1(a^*/c_1)^{\sigma_a}e^{\lambda_a+(1-\sigma_a)c_2} + \beta_s)\}
\end{align*}
and 
$Y_0(U)|\{S=0, A=a^*\}\overset{d}{=}\mathbf{1}\{U_Y < \mathrm{expit}(\beta_0 + \beta_a a^*)\}$. Thus the bank's decision $Y$ would be counterfactually fair if $\sigma_a=1$, $\lambda_a=0$ and $\beta_s=0$.

\section{Preprocessing, Learning, and Testing}\label{sec:algorithm}
Define a preprocessing procedure $\mathcal{P}^{\mathcal{D}}(s, a):\mathcal{S}\times\mathcal{A}\to\mathcal{A}'$ to be a function that maps attributes $(s, a)$ to the processed attributes $a'$ given the training data $\mathcal{D}$. Here we consider two such procedures. Denote $\p_n(S = s)$ as the empirical p.m.f.~of $S$ and $\E_n(A | S = s)$ as the empirical conditional mean of $A$ given $S$ learned from data $\mathcal{D}$.
\begin{definition}[Orthogonalization]\label{def:orth}
An orthogonalization procedure $\mathcal{P}_{O}^{\mathcal{D}}$ is a preprocessing procedure such that
\begin{equation*}
    \mathcal{P}_{O}^{\mathcal{D}}(s^*, a^*) = \sum_s \hat{a}(s) \p_n(S = s),
\end{equation*}
where $\hat{a}(s) = a^* - \E_n(A | S = s^*) + \E_n(A | S = s), \forall s \in \mathcal{S}$.
\end{definition}
It is easy to see that $\mathcal{P}_{O}^{\mathcal{D}}(s^*, a^*) = a^* - \E_n(A | S = s^*) + \E_n(A)$ is a one-to-one function of $a^*$ for any fixed $s^*$. Denote $\hat{F}_{js}(x) = \p_n(A_j \le x | S=s)$ as the empirical marginal cumulative distribution function (CDF) of the $j$th element of the non-sensitive attributes given the sensitive attribute $S=s$. Define its inverse as
\begin{equation}\label{eq:invcdf}
	\hat{F}_{js}^{-1}(z) = \inf\{x: \p_n(A_j \le x | S=s)\ge z\}.
\end{equation}
\begin{definition}[Marginal Distribution Mapping]\label{def:margin}
A marginal distribution mapping $\mathcal{P}_{M}^{\mathcal{D}}$ is a preprocessing procedure such that
\begin{equation*}
	\mathcal{P}_{M}^{\mathcal{D}}(s^*, a^*) = \sum_s \hat{a}(s) \p_n(S = s),
\end{equation*}
where the $j$th element of $\hat{a}(s)$ is $[\hat{a}(s)]_j = \hat{F}_{js}^{-1}(\hat{F}_{js^*}([a^*]_j)$ for $j = 1, \cdots, d$.
\end{definition}
	
Let $\mathcal{P}$, $\mathcal{P}_{O}$, and $\mathcal{P}_{M}$ denote the population level preprocessing procedure corresponding to $\mathcal{P}^{\mathcal{D}}$, $\mathcal{P}_{O}^{\mathcal{D}}$, and $\mathcal{P}_{M}^{\mathcal{D}}$, respectively. It is obvious that $\mathcal{P}_{O}(s^*, a^*) = a^* - \E(A | S = s^*) + \E(A)$ is still a one-to-one function of $a^*$ for any fixed $s^*$, and the $j$th element of $\mathcal{P}_{M}(s^*, a^*)$ is
\begin{equation*}
	[\mathcal{P}_{M}(s^*, a^*)]_j = \sum_{s} F_{js}^{-1}(F_{js^*}([a^*]_j) \p(S = s),
\end{equation*}
where $F_{js}$ is the marginal CDF of the $j$th element of $A$ given $S=s$ and $F_{js}^{-1}$ is defined similarly to (\ref{eq:invcdf}) but replacing $\p_n$ with $\p$. It can be seen that if $A_j$ is a discrete variable, then $F_{js}^{-1}(F_{js^*}(x))$ is strictly increasing for $s=s^*$; and if $A_j$ is a continuous variable, then $F_{js}^{-1}(F_{js^*}(x))$ may not be strictly increasing when $F_{js^*}(x)$ is constant on some interval of $x$. Therefore $\mathcal{P}_{M}(s^*, a^*)$ is only a one-to-one function of $a^*$ for any fixed $s^*$ when the marginal CDF of each continuous element in $A$ given $S=s^*$ is strictly increasing.
	
\subsection{Fair Learning Algorithm}
Besides preprocessing procedures, we also have different choices of learners. A Fairness-Through-Unawareness (FTU) predictor $f_{FTU}(a)$ only uses the non-sensitive attributes $A$ to predict the conditional mean of $Y$. A Machine Learning predictor $f_{ML}(s, a)$ uses both the sensitive and non-sensitive attributes to predict $\E(Y|S, A)$. An Averaged Machine Learning (AML) predictor $f_{AML}(a) = \sum_s f_{ML}(s, a) \p_n(S=s) ds$. Note that we still need to train the ML predictor to obtain the AML predictor, but it only needs the non-sensitive attributes as its input when making a prediction since the sensitive attributes are averaged out. Algorithm \ref{alg:1} could use any learner $f \in \{f:\mathcal{A} \to [0,1]\}$ to learn the decisions from the processed data, and we would consider the FTU and AML learners in our numerical studies.

\begin{algorithm}[!htbp]
    \caption{Fair Learning through dAta Preprocessing (FLAP)}\label{alg:1}
	\begin{algorithmic}
	\STATE {\bfseries Input:} Training data $\mathcal{D}$, preprocessing procedure $\mathcal{P}^{\mathcal{D}}$, learner $f$, test attributes $(s, a)$
	\FOR{$(s_i, a_i, y_i)$ {\bfseries in} $\mathcal{D}$}
	\STATE $a_i' = \mathcal{P}^{\mathcal{D}}(s_i, a_i)$
	\ENDFOR
	\STATE Create the processed data $\mathcal{D}' = \{(s_i, a_i', y_i)\}_{i=1}^n$
	\STATE Learn predictor $f$ from $\mathcal{D}'$
	\STATE Calculate $a' = \mathcal{P}^{\mathcal{D}}(s, a)$
	\STATE Draw $\hat{Y}$ from $\mathrm{Bernoulli}(f(a'))$
	\STATE {\bfseries Output:} $\hat{Y}$
	\end{algorithmic}
\end{algorithm}

Apart from the structural assumptions made in Figure \ref{fig:M}, extra conditions of the structural equation $f_A(s, u_A)$ must be satisfied for the preprocessing method to work.
	
\begin{assumption}[Strong non-sensitive]\label{as:strong_nos}
	The partial derivative $\frac{\partial}{\partial u_A}f_A(s, u_A)$ does not involve $s$.
\end{assumption}
	
\begin{assumption}[Weak non-sensitive]\label{as:weak_nos}
	The sign of $\frac{\partial}{\partial u_A}f_{A_j}(s, u_A)$ does not change with $s$ for all $u_A$ and all $j = 1, \cdots, d$.
\end{assumption}
These two conditions describe the relationship between the sensitive and non-sensitive attributes. Condition \ref{as:weak_nos} is weaker than Condition \ref{as:strong_nos}. 
For example, an additive model $f_A(s, u_A) = \beta_0+\beta_1 s + \beta_2 u_A$ satisfies both conditions, while an interaction model $f_A(s, u_A) = \beta_0+\beta_1 s + \beta_2 u_A + \beta_3 s u_A$ does not satisfy Condition \ref{as:strong_nos} but will satisfy Condition \ref{as:weak_nos} if $\beta_2 + \beta_3 s$ is greater than (or less than, or equal to) zero for all $s$. In our running example, 
$\frac{\partial}{\partial u_A}f_{A}(s, u_A) = c_1c_3 \sigma_a^s\exp\{c_2 + \lambda_a s + c_3 \sigma_a^s u_A\}>0$
for $s=0, 1$. So it meets Condition \ref{as:weak_nos} but not Condition \ref{as:strong_nos}. We prove in the following theorem that these conditions, together with the SCM, are sufficient for Algorithm \ref{alg:1} to generate counterfactually fair decisions.
	
\begin{theorem}{}\label{thm:1}
    Given an SCM $M=(U, V, F)$ with structural equations defined in Figure \ref{fig:M} and let $\hat{Y}$ be the output from Algorithm \ref{alg:1}, i.e., $\mathbf{1}\{U_{\hat{Y}}<f(\mathcal{P}^{\mathcal{D}}(S, A))\}$.
	\begin{enumerate}
		\item If the procedure $\mathcal{P}_{O}^{\mathcal{D}}$ is adopted, $\hat{Y}$ is counterfactually fair under Condition \ref{as:strong_nos}.
		\item If the procedure $\mathcal{P}_{M}^{\mathcal{D}}$ 
		is adopted, $\hat{Y}$ is counterfactually fair under Condition \ref{as:weak_nos}.
	\end{enumerate}
\end{theorem}

\begin{proof}
We prove the theorem for a general class of learners $\{f:\mathcal{A} \to [0,1]\}$ that only take the non-sensitive attribute $a$ as the input. Clearly, both $f_{FTU}$ or $f_{AML}$ belong to this class. We follow the Abduction-Action-Prediction steps in Theorem 7.1.7 \cite{pearl2009causality} to evaluate the conditional expectation of $\hat{Y}_{s'}(U)$ given the evidence $S = s^*, A = a^*$,
\begin{equation*}
    \E(\hat{Y}_{s'}(U)|S = s^*, A = a^*)=\int f(\mathcal{P}^{\mathcal{D}}(s', f_A(s', u)))p_{U_A|S,A}(u|S = s^*, A = a^*)du,
\end{equation*}
where $p_{U_A|S,A}(u|s^*, a^*)$ denotes the conditional density of $U_A$ given $S=s^*$ and $A=a^*$. If $\mathcal{P}^{\mathcal{D}}(s', f_A(s', u))$ does not depend on $s'$, so will $\E(\hat{Y}_{s'}(U)|S = s^*, A = a^*)$ and we will have
\begin{equation*}
    \E(\hat{Y}_{s'}(U)|S = s^*, A = a^*) = \E(\hat{Y}_{s^*}(U)|S = s^*, A = a^*).
\end{equation*}
Note that $\mathcal{P}^{\mathcal{D}}(s', f_A(s', u)) = \sum_s \hat{a}(s) \p_n(S=s)$ for both the preprocessing procedures we are considering. Therefore, it suffices to show that $\hat{a}(s)$ does not depend on $s'$.

First, consider the Orthogonalization procedure $\mathcal{P}_{O}^{\mathcal{D}}$ where 
\begin{align*}
\hat{a}(s) =& f_A(s', u) - \E_n(A | S = s') + \E_n(A | S = s)\\
=& f_A(s', u) - \E(A | S = s') + \E_n(A | S = s) - (\E_n - \E)(A | S = s').
\end{align*}
Note that $A | \{S = s'\} =  f_A(s', U_A)$ and the first order Taylor expansion of $\E(A | S = s')$ is
\begin{equation*}
    \E\left(f_A(s', u) + \frac{\partial}{\partial u}f_A(s, u)\bigg|_{s = s', u = u'}(U_A - u)\right) = f_A(s', u) + \frac{\partial}{\partial u}f_A(s, u)\bigg|_{s = s', u = u'} \E(U_A - u)
\end{equation*}
for some $u'$ between $u$ and $U_A$. By Condition 1
\begin{equation*}
    \hat{a}(s) = \frac{\partial}{\partial u}f_A(s, u)\bigg|_{s = s^*, u = u'} \E(u - U_A) + \E_n(A | S = s) + o_{\p}(n)
\end{equation*}
and thus it does not depend on $s'$.

Second, consider the Marginal Distribution Mapping procedure $\mathcal{P}_{M}^{\mathcal{D}}$. Let $f_{A_j}(s, u) = e_j^Tf_{A}(s, u)$ where $e_j$ is a $d$-dimensional vector with the $j$th element being one and all other elements being zeros. The $j$th element of $\hat{a}(s)$ is $[\hat{a}(s)]_j = \hat{F}_{js}^{-1}(\hat{F}_{js'}(f_{A_j}(s', u)))$ for $j = 1, \cdots, d$. Again, the first order Taylor expansion of $f_{A_j}(s', U_A)$ gives
\begin{align*}
    \hat{F}_{js'}(f_{A_j}(s', u))
    =& \p_n(A_j \le f_{A_j}(s', u) | S = s')\\
    =& \p(f_{A_j}(s', U_A) \le f_{A_j}(s', u))
    + (\p_n - \p)(A_j \le f_{A_j}(s', u) | S = s')\\
    =& \p\bigg(f_{A_j}(s', u)
    + \frac{\partial}{\partial u}f_{A_j}(s, u)\bigg|_{s = s', u = u'}(U_A - u) < f_{A_j}(s', u)\bigg)
    + o_{\p}(n)
\end{align*}
for some $u'$ between $u$ and $U_A$. Under Condition 2,
\begin{equation*}
    \hat{F}_{js'}(f_{A_j}(s', u))
    = \p\left(\mathrm{sign}\left(\frac{\partial}{\partial u}f_{A_j}(s, u)\bigg|_{s = s^*, u = u'}\right)(U_A - u) < 0\right)
    + o_{\p}(n)
\end{equation*}
does not depend on $s'$ and hence $a(s)$ is a function of $s$ and $u$ alone.  
\end{proof}

The intuition is that the FLAP algorithm learns the decision from processed data only, and the processed data contain no sensitive information since the preprocessing procedure can remove $A$'s dependence on $S$ under the non-sensitive condition.

Theorem \ref{thm:1} identifies the conditions for achieving counterfactual fair decisions under a certain SCM. However, the SCM is often unidentifiable given the observational data, i.e., two different SCMs could generate the data with the same joint distribution. For example, let $h$ be a one-to-one function and $U_A' = h(U_A)$. Since we cannot identify if $U_A$ or $U_A'$ are the true exogenous variables, the structural equation of $A$ can be either $f_A(S, U_A)$ or $f_A'(S, U_A') = f_A(S, h^{-1}(U_A'))$. If we are given the freedom to reparameterize the SCM, which is true in most real-world applications, it will be much easier to achieve counterfactual fairness. For illustration, consider the case where $A$ is a continuous random variable, let $U_A'$ be a uniform random variable and $f_A'$ be the inverse conditional CDF of $A$ given $S$. Condition \ref{as:weak_nos} will hold true if $f_A'$, or equivalently, the conditional CDF of $A$ given $S$ is strictly increasing. Then we can apply the Marginal Distribution Mapping procedure to learn counterfactually fair decisions. It is easy to check that this reparameterization trick also works for discrete random variables and random vectors where each element is either continuous or discrete. The only exception where Condition \ref{as:weak_nos} does not hold even with reparameterization is when $A$ contains mixtures of continuous and discrete random variables. One example is test score where the distribution below the maximum score is continuous but there is also a positive probability of getting 100. In this sense, we can learn counterfactually fair decisions for most common types of non-sensitive attributes using the Marginal Distribution Mapping preprocessing procedure.

\subsection{Test for Counterfactual Fairness}
Data preprocessing not only allows us to learn a counterfactually fair decision but also enables us to test if the decisions made in the original data are fair. When Condition \ref{as:strong_nos} holds, we can use the data processed by the orthogonalization procedure to test fairness. When the strong condition does not hold but Condition \ref{as:weak_nos} is satisfied, we need an extra condition to utilize the marginal distribution mapping procedure for fairness testing.

\begin{assumption}\label{as:non_constant_cdf}
	The conditional marginal CDF $F_{js}(x)$ is strictly increasing for all such $j$ that $A_j$ is continuous and all $s\in\mathcal{S}$.
\end{assumption}

In other words, each non-sensitive attributes $A_{j}$ should be either a discrete random variable or a continuous one with non-zero density on $\R$. This condition ensures that $\mathcal{P}_{M}(s^*, a^*)$ is a one-to-one function as discussed earlier. With these conditions, we can establish the equivalence between CF and the conditional independence of decision and sensitive information given the processed non-sensitive information.
	
\begin{theorem}{}\label{thm:2}
Consider the original decision $Y$:
	\begin{enumerate}
		\item Under Condition \ref{as:strong_nos}, $Y$ is counterfactually fair if and only if $Y\bot S | \mathcal{P}_O(S, A)$.
		\item Under Conditions \ref{as:weak_nos} and \ref{as:non_constant_cdf}, $Y$ is counterfactually fair if and only if $Y\bot S | \mathcal{P}_M(S, A)$.
	\end{enumerate}
\end{theorem}
	
\begin{proof}
The steps of proving the two statements of Theorem 2 are similar. To remove redundancy, we use the notation $\mathcal{P}$ whenever the argument is true for both the preprocessing procedures $\mathcal{P}_O$ and $\mathcal{P}_M$.

First we show that $Y$ is counterfactually fair if $Y\bot S | \mathcal{P}(S, A)$. The posterior mean of the counterfactual $Y_{s'}(U)$ given $S=s^*$ and $A=a^*$ can be evaluated in two steps: first find the conditional distribution of $U=\{U_A, U_S, U_Y\}$, and then calculate the conditional expectation of the counterfactuals from the SCM. Since the effect of $U_S$ is blocked by setting $S=s'$ and $U_Y$ is independent of $S$ and $A$, only the distribution of $U_A$ will be affected by the given information and effect the counterfactuals $Y_{s'}(U)$.
\begin{equation}\label{eq:fu}
    \E(Y_{s'}(U)|S=s^*, A=a^*)
    =\int \E(f_Y(s', f_A(s', u), U_Y) p_{U_A|S,A}(u|S=s^*, A=a^*)du.
\end{equation}
Under the SCM, $\E(f_Y(s', f_A(s', u), U_Y)$ is the same as the expectation of the observed decision $Y$ given the attributes $S=s', A=f_A(s',u)$. Therefore (\ref{eq:fu}) is equal to
\begin{align}
    & \int \E(Y | S = s', A=f_A(s',u)) p_{U_A|S,A}(u|S=s^*, A=a^*)du\label{eq:u2a}\\
    =& \int \E(Y | S = s', \mathcal{P}(S, A)=\mathcal{P}(s', f_A(s',u))) p_{U_A|S,A}(u|S=s^*, A=a^*)du\label{eq:a2psa}\\
    =& \int \E(Y |\mathcal{P}(S, A)=\mathcal{P}(s', f_A(s',u))) p_{U_A|S,A}(u|S=s^*, A=a^*)du\label{eq:condind}\\
    =& \int \E(Y |\mathcal{P}(S, A)=\mathcal{P}(s^*, f_A(s^*,u))) p_{U_A|S,A}(u|S=s^*, A=a^*)du\label{eq:pd}\\
    =&\E(Y_{s^*}(U)|S=s^*, A=a^*).
\end{align}
Equation (\ref{eq:a2psa}) replaces the condition $A=f_A(s',u)$ with $\mathcal{P}(S, A)=\mathcal{P}(s', f_A(s',u))$ because $\mathcal{P}_{O}(S, A)$ is a one-to-one function of $A$ given $S$ and $\mathcal{P}_{M}(S, A)$ is also a one-to-one function of $A$ given $S$ under Condition 3. Equation (\ref{eq:condind}) is due to the conditional independence of $Y$ and $S$, and (\ref{eq:pd}) uses the result that $\mathcal{P}(s', f_A(s', u)) = \mathcal{P}(s^*, f_A(s^*, u))$, which can be shown following the proof of Theorem 1. 
Repeat the steps (\ref{eq:fu}) to (\ref{eq:pd}) and we shall get the same result for $\E(Y_{s'}(U)|S=s^*, A=a^*)$. Note that both $Y_{s'}(U)$ and $Y_{s^*}(U)$ are binary random variables, therefore the equivalence in expectation implies that
\begin{equation*}
    Y_{s'}(U)|\{S = s^*, A = a^*\} \overset{d}{=} Y_{s^*}(U)|\{S = s^*, A = a^*\}.
\end{equation*}
The above result holds for any $s', s^* \in \mathcal{S}$, so the definition of counterfactual fairness is satisfied.

Next we show that $Y\bot S | \mathcal{P}(S, A)$ if $Y$ is counterfactually fair. The counterfactual fairness of $Y$ implies
\begin{equation}\label{eq:cfeq}
   \E[f_Y(s', f_A(s', U_A), U_Y) | S=s^*, A=a^*] = \E[f_Y(s^*, f_A(s^*, U_A), U_Y) | S=s^*, A=a^*]. 
\end{equation}
Let $a' = \mathcal{P}(s^*, a^*)$, then
\begin{equation}\label{eq:u1}
    (U_A, U_Y) | \{S=s^*, A=a^*\}
    \overset{d}{=} (U_A, U_Y) | \{S=s^*, \mathcal{P}(s^*, A)=a'\}
\end{equation}
since $\mathcal{P}(s^*, a^*)$ is a one-to-one function of $a^*$ for each $s^*$.

Using the Bayesian formula, the posterior density of $U_A$ is
\begin{equation}
    p_{U_A|S,\mathcal{P}}(u|S=s^*, \mathcal{P}(s^*, A)=a')
    =\frac{p_{U_A}(u)\p(S=s^*)p_{\mathcal{P}|S,U_A}(a'|S=s^*, U_A=u)}{\p(S=s^*)\int p_{U_A}(u)p_{\mathcal{P}|S,U_A}(a'|S=s^*, U_A=u)du},\label{eq:bayes}
\end{equation}
where $p_{U_A|S,\mathcal{P}}(u|s^*, a')$ denotes the conditional density of $U_A$ given $S=s^*$ and $\mathcal{P}(s^*, A)=a'$, $p_{U_A}(u)$ denotes the prior density of $U_A$, and $p_{\mathcal{P}|S,U_A}(a'|S=s^*, U_A=u)$ denotes the conditional density of $\mathcal{P}(s^*, A)$ given $S=s^*$ and $U_A=u$. As a density function, (\ref{eq:bayes}) is proportional to its kernel $p_{U_A}(u)p_{\mathcal{P}|S,U_A}(a'|S=s^*, U_A=u)$, which equals $p_{U_A}(u)p_{\mathcal{P}|S,U_A}(a'|S=s', U_A=u)$ because $\mathcal{P}(S, A)$ does not depend on $S$ when $U_A$ is given as shown in the proof of Theorem 1. Repeating the steps and we can show that the posterior density of $U_A$ given $S=s', \mathcal{P}(s', A)=a'$ is also proportional to $p_{U_A}(u)p_{\mathcal{P}|S,U_A}(a'|S=s', U_A=u)$. Together with the assumption in the SCM that $U_Y$ is independent of $S, \mathcal{P}(S, A)$, we have
\begin{equation}\label{eq:u2}
(U_A, U_Y) | \{S=s', \mathcal{P}(s', A)=a'\} \overset{d}{=} (U_A, U_Y) | \{S=s^*, \mathcal{P}(s^*, A)=a'\}.
\end{equation}
The intuition here is that if the processed non-sensitive data are equal, then they provide the same information about $U_A$ regardless of the sensitive information in the original data. Substituting the conditions $\{S=s^*, A=a^*\}$ in (\ref{eq:cfeq}) with the equivalent conditions in (\ref{eq:u1}) and (\ref{eq:u2}) gives
\begin{equation}\label{eq:counterfactualeq}
\begin{split}
    &\E[f_Y(s', f_A(s', U_A), U_Y) | S=s', \mathcal{P}(s', A)=a']\\
    =&\E[f_Y(s^*, f_A(s^*, U_A), U_Y) | S=s^*, \mathcal{P}(s^*, A)=a'].
\end{split}
\end{equation}
Under the SCM and structural equations defined in Figure 1, (\ref{eq:counterfactualeq}) implies
\begin{equation}\label{eq:condeq}
    \E[Y | S=s', \mathcal{P}(S, A)=a'] = \E[Y | S=s^*, \mathcal{P}(S, A)=a'].
\end{equation}
Since (\ref{eq:condeq}) holds for any $s', s^*\in\mathcal{S}$, it yields that
\begin{equation*}
    \E[Y|S, \mathcal{P}(S, A)] = \E[Y|\mathcal{P}(S, A)]
\end{equation*}
and hence $Y\bot S | \mathcal{P}(S, A)$ for binary $Y$.
\end{proof}

Theorem \ref{thm:2} allows us to test CF using any well-established conditional independence test. In practice, given a decision dataset $\mathcal{D} = (s_i, a_i, y_i)_{i=1}^n$, we can obtain the empirical processed non-sensitive attributes $\mathcal{P}^{\mathcal{D}}(s_i, a_i)$ and test if $Y\bot S | \mathcal{P}^{\mathcal{D}}(S, A)$. If the p-value of the test is small enough for us to reject the conditional independence hypothesis, then the original decision is probably biased and algorithms such as FLAP should be used to learn fair decisions.
	
\section{Numerical Studies}
In this section, we compare the decisions made by different algorithms in terms of fairness and accuracy using simulated and real data, and also investigate the empirical performance of the fairness test using simulated data with small sample sizes. We consider three cases for generating the simulation data. The first one is Example \ref{ex:1} and the second one is a multivariate extension of it.

\begin{example}\label{ex:2}
The bank now collects the race $S$, education year $E$ and annual income $A$
information from loan applicants. There are three possible race groups $\mathcal{S} = \{0,1,2\}$ and $S = \mathbf{1}\{U_S > 0.76\} + \mathbf{1}\{U_S > 0.92\}$, meaning that a random applicant could be from the majority race group ($0$) with probability $0.76$, or from the minority group $1$ or $2$ with probability $0.16$ or $0.08$. Let $U_E$ be a standard normal random variable and $\mu_E = \lambda_{e0} + \mathbf{1}\{S = 1\}\lambda_{e1} + \mathbf{1}\{S = 2\}\lambda_{e2}$,  the education year is $E = \max\{0, \mu_E + 0.4\mu_EU_E\}$. Let $\mu_A = \log(\lambda_{a0} + \mathbf{1}\{S = 1\}\lambda_{a1} + \mathbf{1}\{S = 2\}\lambda_{a2})$, the annual income is $A = \exp\{\mu_A + 0.4\mu_EU_E + 0.1U_A\}$. The decision of the bank is modeled as
\begin{equation*}
        Y = \mathbf{1}\{U_Y < \mathrm{expit}(\beta_{0} + \mathbf{1}\{S = 1\}\beta_{1} +\mathbf{1}\{S = 2\}\beta_{2} + \beta_a A + \beta_e E)\}.
\end{equation*}
\end{example}

Here $\lambda_{e0}, \lambda_{e1}$, and $\lambda_{e2}$ decide the mean education year of the three race groups. $\lambda_{a0}, \lambda_{a1}$, and $\lambda_{a2}$ decide the median annual income. The annual income and the education year are positively correlated through $U_E$. $\beta_1$ and $\beta_2$ characterize the direct effect of the race information while the $\lambda$'s indicate the indirect effect together with $\beta_e$ and $\beta_a$. In this example, neither of Conditions \ref{as:strong_nos} and \ref{as:weak_nos} holds if $\beta_e$ and $\lambda_{e1}$ and/or $\lambda_{e2}$ are not zero due to the maximum operator in $f_E$. Even if $\lambda_{e1} = \lambda_{e2} = 0$, only the weaker Condition \ref{as:weak_nos} will hold due to the same reason for Example \ref{ex:1}.

The third example is a replica of the admission example constructed by \citet{wang2019equal}. 

\begin{example}\label{ex:3}
The admission committee of a university collects the gender $S$ and test score $T$ information from applicants. The gender is simulated from $S = \mathbf{1}\{U_S < 0.5\}$, where $S = 1$ for male and $S = 0$ for female. Let $U_T \sim \mathrm{Uniform}(0, 1)$ and we generate the test score as $T = \min\{\max\{0, \lambda S + U_T\}, 1\}$. The decision of the committee is 
$$
Y = \mathbf{1}\{U_Y < \mathrm{expit}(\beta_{0} + \beta_t T + \beta_s S)\}.
$$
\end{example}

It is worth noting that Example \ref{ex:3} also does not satisfy either of Conditions \ref{as:strong_nos} and \ref{as:weak_nos} due to the cutoff in the test score. There will be a positive probability ($\lambda$ to be exact) of seeing male students with scores equal to $1$ if $\lambda > 0$. Check that
$$
\frac{\partial}{\partial u_T} f_T(s, u_T) = \begin{cases} 1, & 0 < u_T < 1 - \lambda s\\ 0, & 1 - \lambda s < u_T < 1\end{cases}
$$
and we can see that its sign does change with $s$ for any fixed $u_T$. Therefore, neither of the proposed preprocessing methods can achieve CF in theory.
	
The parameters chosen for these examples are presented in Section 2 of the supplementary material.
	
\subsection{Fairness evaluation}
We compare our FLAP algorithm with 
\begin{enumerate}
	\item ML: the machine learning method using both sensitive and non-sensitive attributes without preprocessing, which is a logistic regression of $Y$ on $S$ and $A$;
	\item FTU: the Fairness-Through-Unawareness method which fits a logistic model of $Y$ on non-sensitive attributes $A$ alone without preprocessing;
	\item FL: the FairLearning algorithm in \citet{kusner2017counterfactual};
	\item AA: the Affirmative Action algorithm in \citet{wang2019equal}.
\end{enumerate}
All these methods can output a predicted score $p$ given the training data $\mathcal{D}$ and test attributes $(s, a)$, denoted $p(s, a; \mathcal{D})$ and draw the random decision $\hat{Y}$ from $\mathrm{Bernoulli}(p(s, a; \mathcal{D}))$. For ML method, $p(s, a; \mathcal{D}) = f_{ML}(s, a)$; for FTU method, that is $f_{FTU}(a)$. We denote the predicted scores of the FairLearning and AA algorithms as $f_{FL}(s, a; \mathcal{D})$ and $f_{AA}(s, a; \mathcal{D})$, respectively. For our FLAP method, we use the marginal distribution mapping procedure and try both the AML and the FTU learners described in Section \ref{sec:algorithm} and name the methods as FLAP-1 and FLAP-2. Their predicted scores are $f_{AML}(\mathcal{P}_{M}^{\mathcal{D}}(s, a))$ and $f_{FTU}(\mathcal{P}_{M}^{\mathcal{D}}(s, a))$, respectively. We use the test accuracy to measure the prediction performance and consider two metrics for measuring the counterfactual fairness. The CF-metric is defined as
\begin{equation*}
	\max_{r,t\in\mathcal{S}}\frac{1}{N}\sum_{i=1}^{N} |p(r, \hat{a}^{\mathcal{D}}_{M}(r, s_i, a_i); \mathcal{D}) - p(t, \hat{a}^{\mathcal{D}}_{M}(t, s_i, a_i); \mathcal{D})|,
\end{equation*}
where $N$ is the size of the test set and $\hat{a}^{\mathcal{D}}_{M}(s, s^*, a^*)$ is defined as $\hat{a}(s)$ in Definition \ref{def:margin}. Note that the CF-metric should be zero when decisions are CF under Condition \ref{as:weak_nos}. This metric is different from the AA-metric proposed by \citet{wang2019equal} in two folds. First, it allows us to consider more than two sensitive groups by taking the maximum of the pairwise difference of predicted scores, but it reduces to the AA-metric for two sensitive groups. Second, we use the marginal distribution mapping method to compute the counterfactual non-sensitive attributes $\hat{a}^{\mathcal{D}}_{M}(s, s^*, a^*)$ had the unit been in a different sensitive group $s$. This ensures that all the derived counterfactual attributes are within the range of observed attribute values. In comparison, \citet{wang2019equal} use the orthogonalization method to compute the counterfactual attributes and thus a female student having test score $0.98$ would have a counterfactual score of $1.48$ had she been a male if the male mean test score is $0.5$ higher than female. This out-of-range counterfactual score is unreasonable and problematic when being used as the input of the score prediction function $p$.

We should note that the CF-metric is a good representative of CF when the weak non-sensitive condition is met. When it is not satisfied, however, a CF decision is not guaranteed to have zero CF-metric. In absence of the non-sensitive conditions, the counterfactual non-sensitive attributes are unidentifiable, and thus \citet{wu2019counterfactual} propose to assess CF using the lower and upper counterfactual fairness bounds, which evaluate to
\begin{equation*}
\begin{split}
    \min_{i, (s', a')\in\mathcal{D}, s' \ne s_i} (p(s', a'; \mathcal{D}) - p(s_i, a_i; \mathcal{D})), \\
    \max_{i, (s', a')\in\mathcal{D}, s' \ne s_i} (p(s', a'; \mathcal{D}) - p(s_i, a_i; \mathcal{D})),
\end{split}
\end{equation*}
under our SCM. Though requiring no assumption on the functional form of $f_A$, these bounds are often too wide to be used for comparing different methods. It can be seen that if the range of the prediction probabilities $p$ is $[0,1]$ for each sensitive group $s$, then the bounds are always $(-1, 1)$. In order for the metric to be more informative, we propose the CF-bound defined as follows. Denote $\mathcal{D}_s$ the subset of the training data $\mathcal{D}$ where the sensitive attribute is $s$. Let $n_s$ be the size of $\mathcal{D}_s$ and $r_s([a]_j)$ be the ascending rank of the $j$th element of $a$ in $\mathcal{D}_s$. The CF-bound is
\begin{equation*}
	\max_{i, s'\in\mathcal{S}\setminus \{s_i\}} |\bar{p}(s', \cdot; \mathcal{D}) - p(s_i, a_i; \mathcal{D})|,
\end{equation*}
where $\bar{p}(s', \cdot; \mathcal{D})$ is the average of predicted scores $p(s', a'; \mathcal{D})$ for a sample of $a'$ randomly selected from the set $\{a: |r_{s'}([a]_j) - r_{s_i}([a_i]_j)| \le \delta n_s, j=1, \cdots, d\}$. The prespecified parameter $\delta$ determines the sampling range of the non-sensitive attributes. If we set $\delta=1$, meaning that $a'$ is sampled from all non-sensitive attributes seen in $\mathcal{D}_{s'}$, then the CF-bound performs like the maximum absolute value of the bounds defined in \citet{wu2019counterfactual} and it is also not very informative. On the other hand, setting $\delta=0$ will only make sense when the the rank of each element of the counterfactual non-sensitive attributes $a'$ in the $s'$ sensitive group is the same as the rank of each element of $a_i$ in the $s_i$ group, which is an assumption that may not hold in real world applications. Therefore, $\delta$ should be chosen from $(0, 1)$ to tell apart different methods and weaken the assumption needed for $f_A$. We use $\delta=0.05$ for the discussions below, and a comparison of the CF-bound results for different $\delta$'s is shown in Section \ref{sec:real}.

In Figure 2b, we set $\beta_0=-1$, $\beta_t=2$, $\beta_s=1$ and increase $\lambda$ from $0$ to $0.8$ to see how the mean difference of test scores affects fairness.

For Example \ref{ex:1}, we choose $c_1=0.01$, $c_2=4$, $c_3=0.2$, fix $\beta_0 = -1$, $\beta_a = 2$, $\beta_s = 1$, and $\lambda_a = 0.5$ while increase $\sigma_a$ from $1$ to $2.8$ to see how the difference in the variation of the non-sensitive attribute between sensitive groups affects fairness. As shown in Figure \ref{fig:metric_ex1}, the AA algorithm which essentially uses the orthogonalization method cannot achieve CF since Condition \ref{as:strong_nos} is not met. However, both FLAP algorithms' CF-metrics are zero when using the marginal distribution mapping preprocessing. The CF-bounds also show that the FLAP methods are the fairest among the methods we consider.
	
\begin{figure}%
    \centering
    \begin{subfigure}{1\columnwidth}
        \includegraphics[width=\columnwidth]{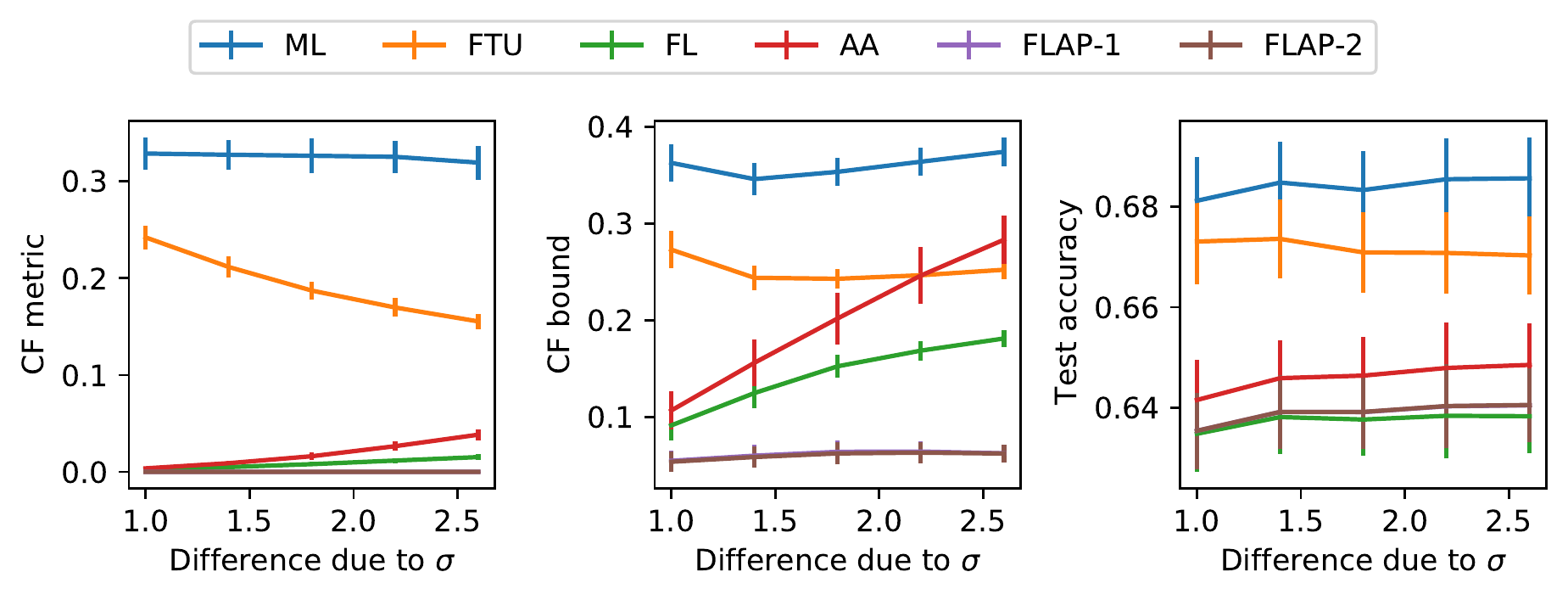}
        \caption{Example \ref{ex:1} with increasing $\sigma_a$.}
        \label{fig:metric_ex1}
    \end{subfigure}\hfill%
	\begin{subfigure}{1\columnwidth}
		\includegraphics[width=\columnwidth]{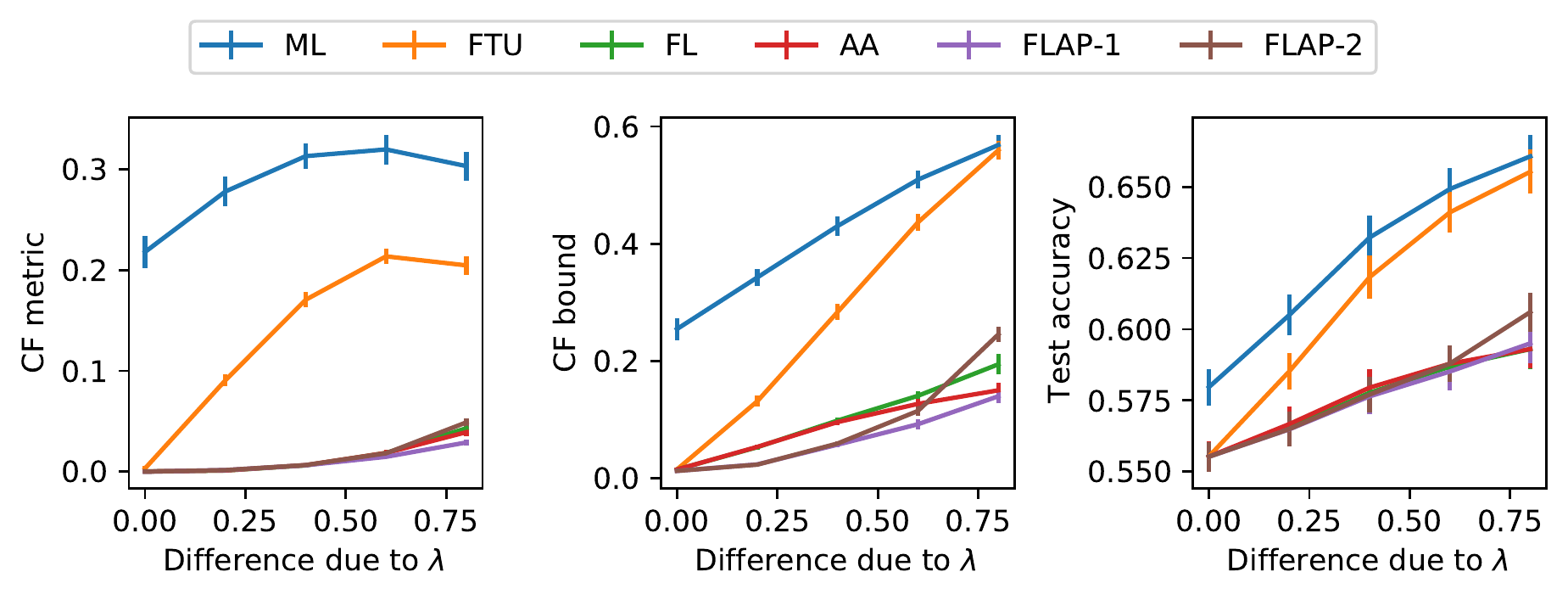}
        \caption{Example \ref{ex:3} with different mean scores by gender.}
        \label{fig:metric_ex3}
	\end{subfigure}%
	 \caption{Comparison of CF-metric, CF-bound and test accuracy of decision making algorithms. The lines and bars are the mean and standard deviation of the results in 100 experiments.}
\end{figure}

\citet{wang2019equal} showed that the AA algorithm can achieve zero AA-metric in Example \ref{ex:3}, but it does not satisfy either of the non-sensitive conditions for achieving CF. In Figure \ref{fig:metric_ex3}, we fix $\beta_0=-1$, $\beta_t=2$, $\beta_s=1$ and increase $\lambda$ from $0$ to $0.8$. It can be seen that all algorithms we consider cannot achieve CF, but the FLAP-1 algorithm still has the lowest CF-metric and CF-bound. There is no significant difference between the accuracy of the FL, AA, and FLAP algorithms in all examples. In general, we expect fairer predictions to have lower accuracy since they correct the discriminatory bias of the original decisions. 

\begin{figure}
    \centering
	\begin{subfigure}{\columnwidth}
		\includegraphics[width=\linewidth]{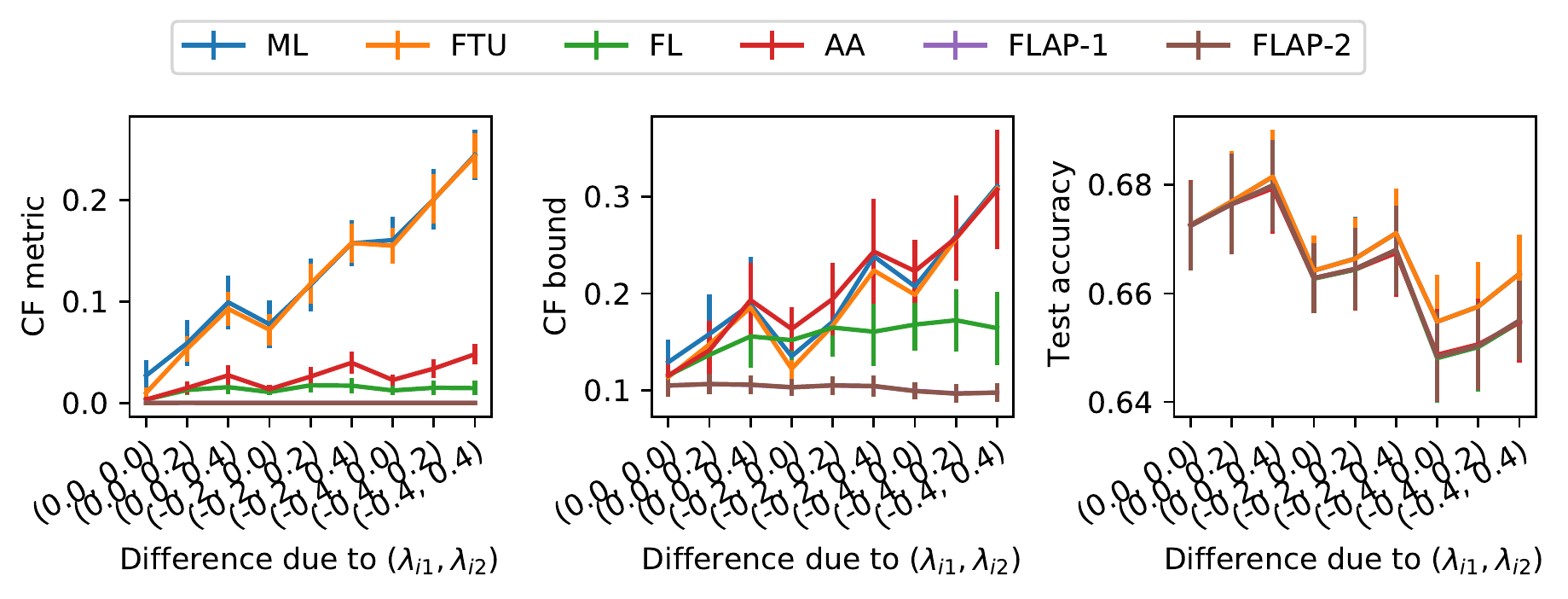}
        \caption{Ex. 2 with different mean income by race.}
        \label{fig:metric_ex2_income}
	\end{subfigure}\hfill%
	\begin{subfigure}{\columnwidth}
        \includegraphics[width=\linewidth]{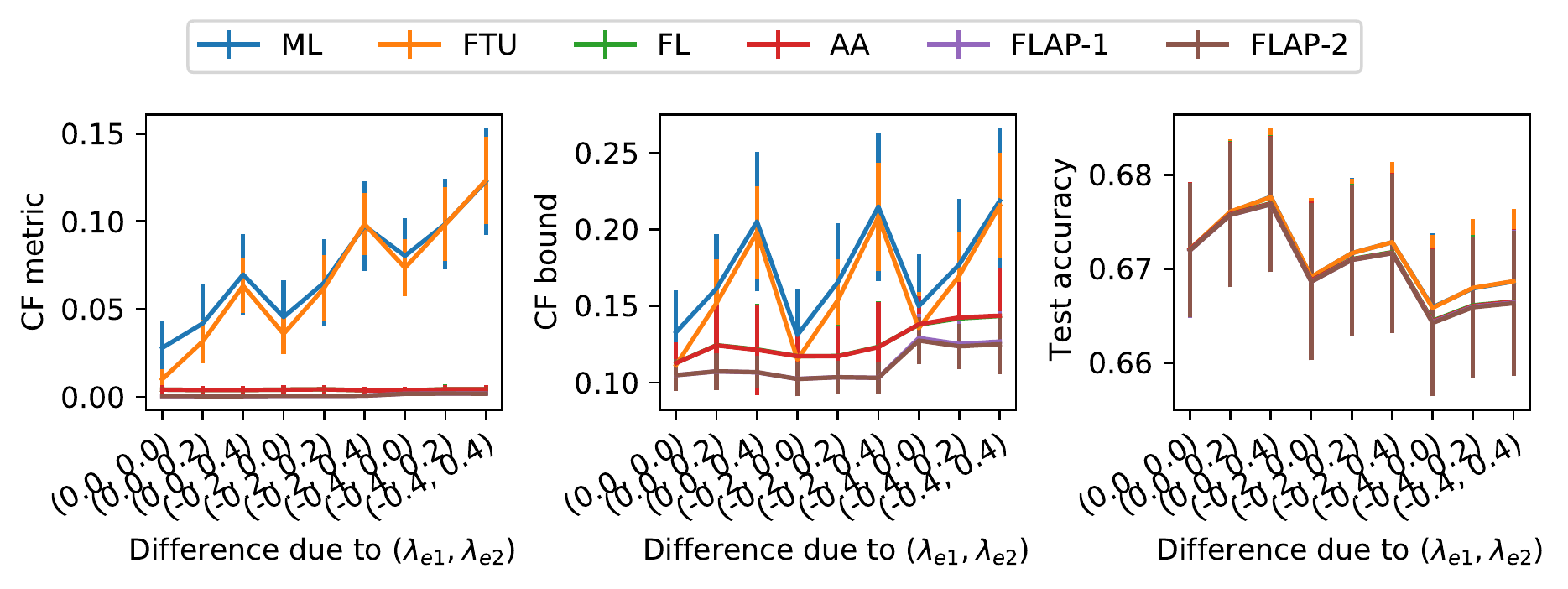}
        \caption{Ex. 2 with different mean education years by race.}
        \label{fig:metric_ex2_edu}
    \end{subfigure}%
	 \caption{Comparison of CF-metric, CF-bound and test accuracy of decision making algorithms}
	 \label{fig:metric_ex2}
\end{figure}

In Example \ref{ex:2}, we choose $\beta_0=-1$, $\beta_1=\beta_2=0$, $\beta_a=1$, $\beta_e=2$, $\lambda_{e0}=1.07$, $\lambda_{i0}=0.58$. In Figure \ref{fig:metric_ex2_income}, we change $(\lambda_{i1}, \lambda_{i2})$ while fix $\lambda_{e1}=0$ and $\lambda_{e2}=0$ to see how the mean difference of income affect fairness. The results are telling the same story as Figure \ref{fig:metric_ex1}: since only the weaker non-sensitive condition is met, the AA-algorithm cannot achieve CF but the FLAP algorithms with marginal distribution mapping procedure can.
	
In Figure \ref{fig:metric_ex2_edu}, we change $(\lambda_{e1}, \lambda_{e2})$ while fix $\lambda_{i1}=0$ and $\lambda_{i2}=0$ to see how the mean difference of education affect fairness. The results are similar to those of Figure \ref{fig:metric_ex3} where all algorithms we consider cannot achieve CF but the FLAP algorithms still have the lowest CF-metric and CF-bound.
	
\subsection{Fairness Test}
The Conditional Distance Correlation (CDC) test~\citep{wang2015conditional} is a well-established non-parametric test for conditional independence. We use it here to illustrate the performance of the fairness test with the three simulated examples. For each example, we use different combinations of parameters to obtain simulated datasets with different fairness levels, which are measured by the CF-metric. A CDC test with a significance level of 0.05 is then conducted to test if $Y\bot S | \mathcal{P}^{\mathcal{D}}(S, A)$ for each dataset. The simulation-test process is repeated 1000 times for each combination of parameters to estimate the power of the test, namely the probability of rejecting the null hypothesis that the decisions are counterfactually fair. The results are summarized in Figure \ref{fig:power}.

\begin{figure*}
    \centering
    \begin{subfigure}{0.5\textwidth}
        \includegraphics[width=\linewidth]{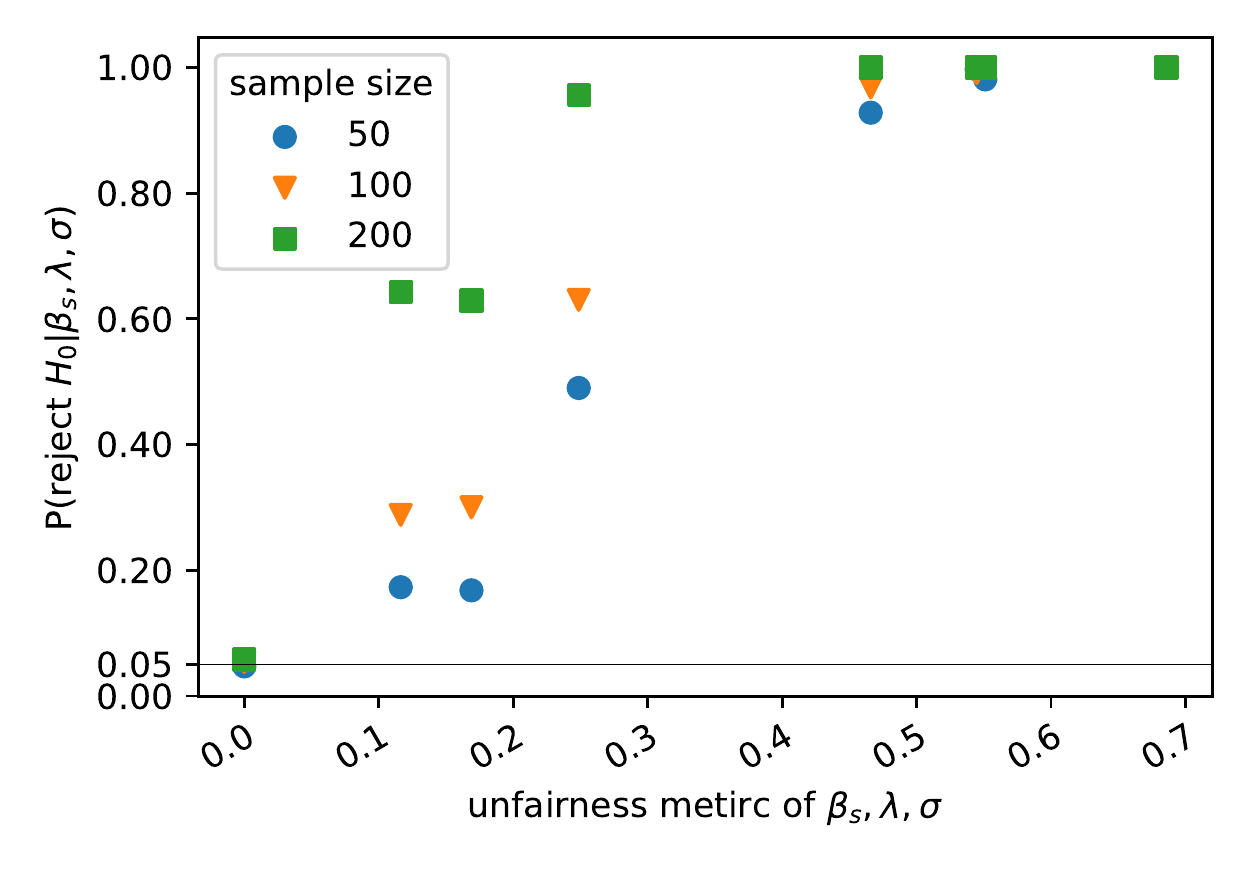}
        \caption{Example \ref{ex:1}.}
        \label{fig:power_ex1}
    \end{subfigure}\hfill%
	\begin{subfigure}{0.5\textwidth}
		\includegraphics[width=\linewidth]{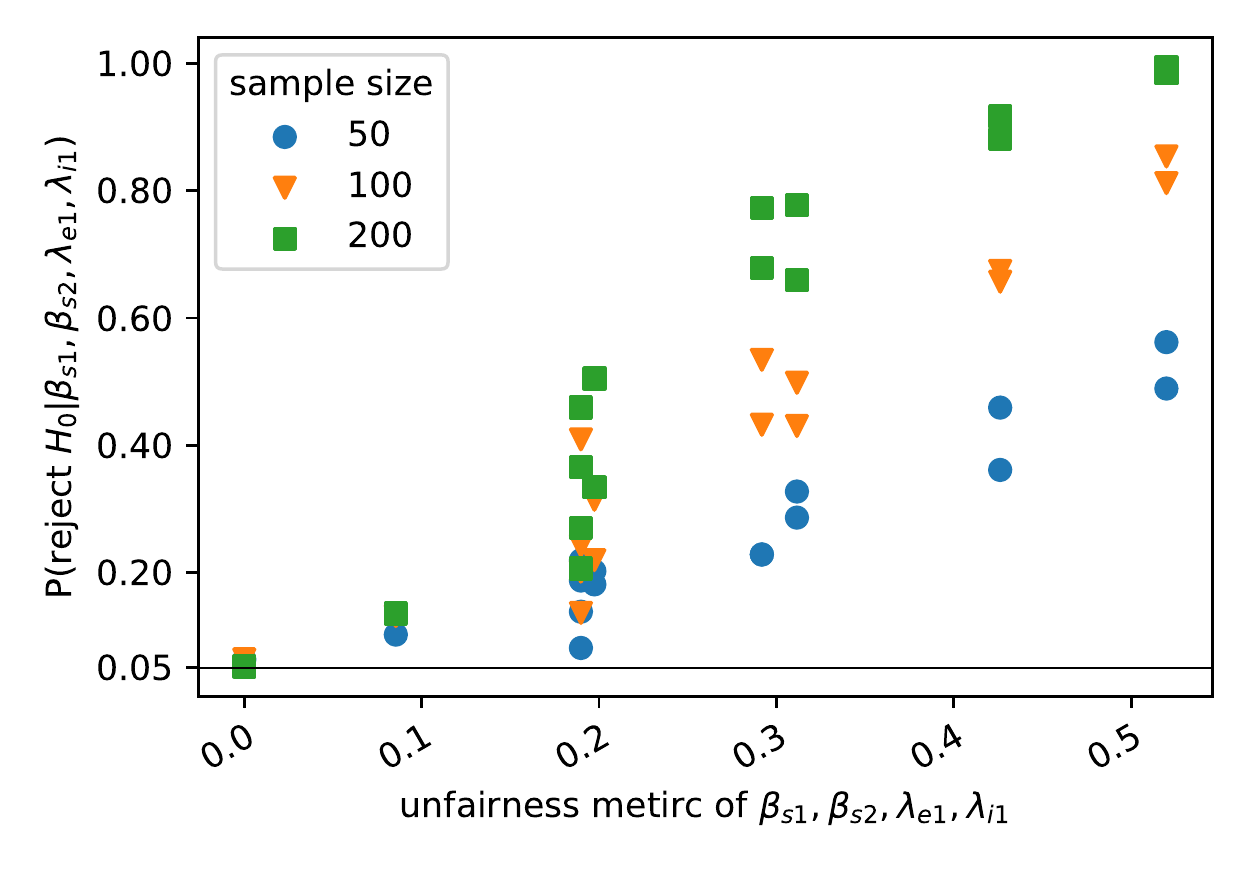}
        \caption{Example \ref{ex:2}.}
        \label{fig:power_ex2}
	\end{subfigure}\\
	\begin{subfigure}{0.5\textwidth}
        \includegraphics[width=\linewidth]{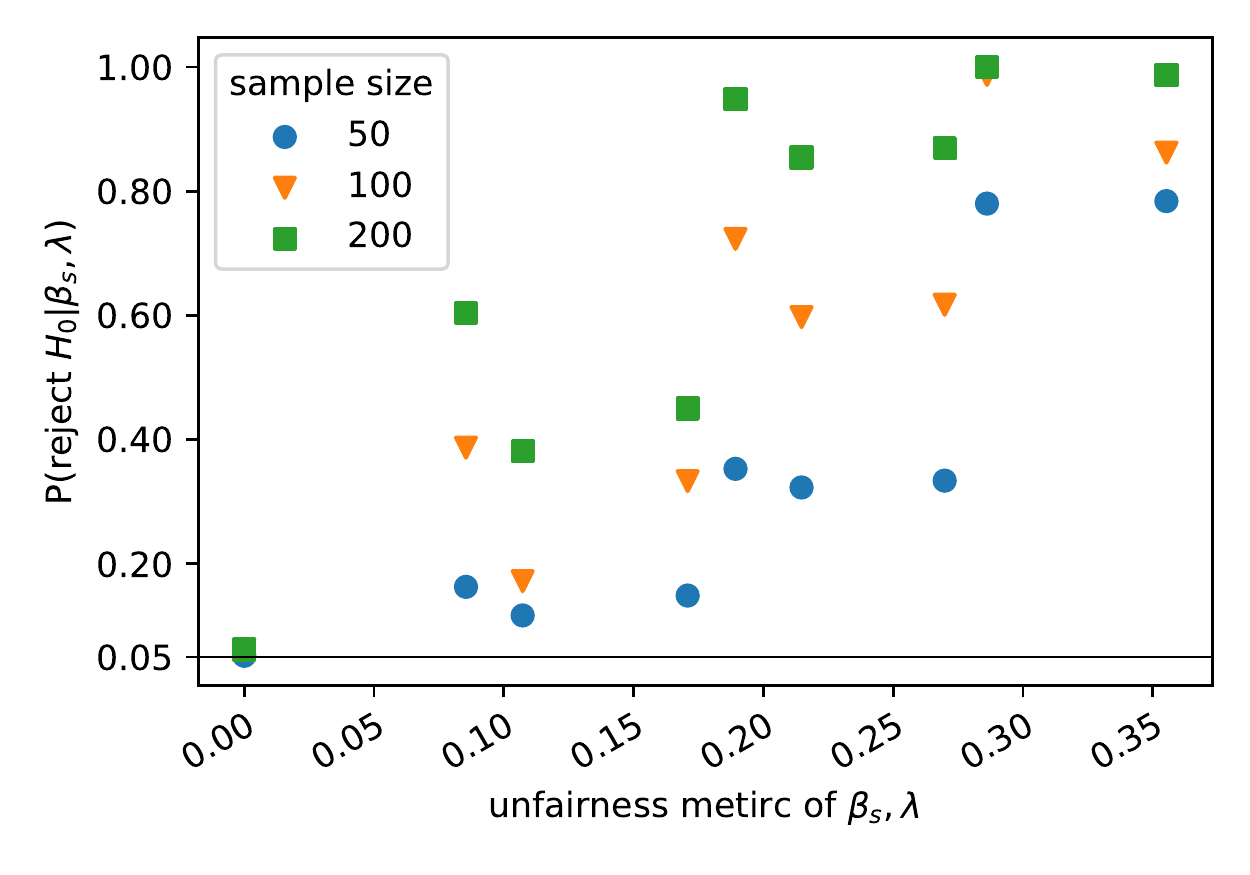}
        \caption{Example \ref{ex:3}.}
        \label{fig:power_ex3}
    \end{subfigure}%
	\caption{Power for testing CF using conditional independent test plot against the CF-metric}
	\label{fig:power}
\end{figure*}

When the decisions are generated fair, which are shown as the points with CF-metrics equal to zero, the type I error rate is around 0.05 for all examples. The power of the test grows as we make the decisions more unfair, or increase the sample size. 

\section{Real Data Analysis}\label{sec:real}
We apply our methods to a loan application dataset from a fintech company, the adult income dataset from UCI Machine Learning Repository\footnote{\url{https://archive.ics.uci.edu/ml/machine-learning-databases/adult/}} and the COMPAS recidivism data from ProPublica\footnote{\url{https://github.com/propublica/compas-analysis}} \citep{angwin2016bias}.

In the loan application case, the fintech lender aims to provide short-term credit to young salaried professionals by using their mobile and social footprints to determine their creditworthiness even when a credit history may not be available. 
To get a loan, a customer has to download the lending app, submit all the requisite details and documentation, and give permission to the lender to gather additional information from her/his smartphone, such as the number of apps, number of calls, and SMSs, and number of contacts and social connections. We obtained data from the lending firm for all loans granted from February 2016 to November 2018. The decisions $Y$ are whether or not the lender approves the loan applications. The attributes are applicants' gender, age, salary, and other information collected from their smartphones. Both gender and age are regarded as sensitive information here and we find that the decisions are made in favor of the senior and female applicants. Since we can only deal with categorical sensitive attributes, we divide the applicants into two age groups by the lower quartile of the age distribution and create a categorical variable $S\in\{0,1,2,3\}$ to denote the group of the applicants: female younger than 28; male younger than 28; female older than 28; and male older than 28. The effective sample size after removing missing values is 203,656.

Non-parametric conditional independence tests will not be efficient for this real case due to the large sample size. Therefore we test the conditional independence of $Y$ and $S$ given $\mathcal{P}_{M}^{\mathcal{D}}(S, A)$ by fitting a simple logistic model for $Y$ with $S$ and $\mathcal{P}_{M}^{\mathcal{D}}(S, A)$ as the explanatory variables and testing if the coefficient of $S$ is significantly different from zero. The p-value of the F-test is almost zero and indicates that the decisions are unfair for applicants in different groups. When other attributes are fixed to their means, the predicted approval probabilities of the four groups from the logistic model are 0.924 (young female), 0.899 (young male), 0.948 (senior female), and 0.946 (senior male), also indicating that the decisions are most in favor of the senior and female applicants.

We then separate the data into a training set of 193,656 samples and a test set of 10,000 samples. The training dataset is used to learn the decisions with different algorithms and the test dataset is used to evaluate the CF-metric, CF-bound, and accuracy. The results are summarized in Table \ref{tab:real}. Since the non-sensitive condition may not be satisfied in this real world application, the CF-bound may be a better indicator of CF than the CF-metric when they disagree with each other. While both the FLAP algorithms using the marginal distribution mapping preprocessing procedure have lower CF-metrics and CF-bounds compared with other algorithms, the FLAP algorithm using the AML learner (FLAP-1) is fairer than the one using the FTU-learner (FLAP-2) as shown by the CF-bound. Their test accuracy is slightly lower than the ML method. Note that in real-world applications, fairer decisions may not have lower accuracy as expected in the simulation studies because we do not have access to all the variables possessed by the original decision-maker. When the original decisions depend on additional information, the FLAP and other fair learning methods may yield predictions closer to or further away from the original decisions, and thus leading to lower or higher accuracy.

\begin{table}[!htbp]
    \centering
    \caption{Comparison of the CF-metric, CF-bound and test accuracy of decision making algorithms on the loan application data. FLAP-1(O) and FLAP-2(O) use the orthogonalization and FLAP-1(M) and FLAP-2(M) use the marginal distribution mapping preprocessing procedure.}
    \begin{tabular}{c|cccc}
    \hline
            &ML     &FTU    &FL     &AA    \\\hline
CF-metric	&0.0392	&0.0130	&0.0011	&0.0011\\
CF-bound    &0.1644 &0.1401 &0.1354 &0.1346\\
accuracy    &\textbf{0.8751}	&0.8749	&0.8742	&0.8734\\
    \hline
            &FLAP-1(O)	&FLAP-2(O)  &FLAP-1(M)  &FLAP-2(M)\\
    \hline
CF-metric	&0.0011	&0.0011 &0.0008 &\textbf{0.0007}\\
CF-bound    &0.1320 &0.1327 &\textbf{0.1217} &0.1227\\
accuracy    &0.8742	&0.8742 &0.8742 &0.8742\\
    \hline
    \end{tabular}
    \label{tab:real}
\end{table}

Table \ref{tab:delta-loan} shows the effect of $\delta$ on the CF-bound. In general, higher $\delta$ means a broader range of $a$'s are considered as the possible counterfactual non-sensitive attributes $a'$ had $S$ been $s'$, and thus the bound will become higher. When $\delta$ is close to 1, the metrics become similar to each other and make it difficult to tell the best method. As shown from the table, the results for different $\delta$'s are mostly consistent with each other.   

\begin{table}[!htbp]
	\centering
	\caption{Comparison of the CF-bound of decision making algorithms on the loan application data using different range parameter $\delta$.}
	\begin{tabular}{c|ccccc}
		\hline
		        &$\delta=0$ &$\delta=.025$  &$\delta=.05$   &$\delta=.1$    &$\delta=1$\\\hline
		ML      &0.1251     &0.1454         &0.1644         &0.1652         &0.2868\\
		FTU     &0.0794     &0.1223         &0.1401         &0.1610         &\textbf{0.2109}\\
		FL      &0.0975     &0.1139         &0.1354         &0.1564         &0.2344\\
		AA      &0.0963     &0.1133         &0.1346         &0.1558         &0.2345\\
		FLAP-1(O)  &0.0964  &0.1218         &0.1320         &0.1465         &0.2362\\
		FLAP-2(O)  &0.0975  &0.1225         &0.1327         &0.1471         &0.2362\\
		FLAP-1(M)  &\textbf{0.0501}&\textbf{0.1060}&\textbf{0.1217}&\textbf{0.1380}&0.2329\\
		FLAP-2(M)  &0.0512     &0.1067         &0.1227         &0.1390      &0.2354\\
		\hline
	\end{tabular}
	\label{tab:delta-loan}
\end{table}

We use the adult income data to predict whether an individual's income is higher than \$50K with information including sex, race, age, workclass, education, occupation, marital-status, capital gain and loss. Sex and race are regarded as sensitive attributes. The training set has 32,561 samples and the test set has 16281 samples. The comparison of the FLAP and other methods are shown in Table \ref{tab:public-1}.

\begin{table}[!htbp]
	\centering
	\caption{Comparison of the CF-metric, CF-bound and test accuracy of decision making algorithms on the adult income data.}
	\begin{tabular}{c|cccc}
		\hline
		            &ML     &FTU    &FL     &AA\\\hline
		CF-metric	&0.2779	&0.2338	&0.0228	&0.0268\\
		CF-bound    &0.9152 &0.8421 &0.7166 &0.7656\\
		accuracy    &0.7612	&0.7604	&0.7594	&\textbf{0.7644}\\
		\hline
		            &FLAP-1(O)	&FLAP-2(O)  &FLAP-1(M)  &FLAP-2(M)\\\hline
		CF-metric	&0.0280	    &0.0228     &\textbf{0.0020}     &0.0022\\
		CF-bound    &0.7357     &\textbf{0.7151}     &0.7721     &0.7470\\
		accuracy    &0.7548	    &0.7594     &0.7570     &0.7599\\
		\hline
	\end{tabular}
	\label{tab:public-1}
\end{table}

The COMPAS (Correctional Offender Management Profiling for Alternative Sanctions) recidivism data contains the demographic data such as sex, age, race, and record data such as priors count, juvenile felonies count, and juvenile misdemeanors count of over 10,000 criminal defendants in Broward County, Florida. The task is to predict whether they will re-offend in two years. According to ProPublica, ``Black defendants were often predicted to be at a higher risk of recidivism than they actually were.'' Here we treat sex and race as sensitive attributes and try to predict recidivism in a counterfactually fair manner. We only use the data for Caucasian, Hispanic, and African-American individuals due to the small sample sizes of other races. The remaining data are divided into a training set of 5,090 samples and a test set of 1697 samples. The results are shown in Table \ref{tab:public-2}.
	
\begin{table}[!htbp]
	\centering
	\caption{Comparison of the CF-metric, CF-bound and test accuracy of decision making algorithms on the COMPAS data.}
	\begin{tabular}{c|cccc}
		\hline
		            &ML     &FTU    &FL     &AA\\\hline
		CF-metric	&0.2274	&0.1406	&0.0054	&0.0060\\
		CF-bound    &0.6087 &0.5892 &0.4956 &0.4961\\
		accuracy    &\textbf{0.5744}	&0.5726	&0.5598	&0.5609\\
		\hline
		            &FLAP-1(O)	&FLAP-2(O)  &FLAP-1(M)  &FLAP-2(M)\\\hline
		CF-metric	&0.0058	    &0.0054     &\textbf{0.0026}     &0.0027\\
		CF-bound    &0.4852     &0.4854     &0.4012     &\textbf{0.4007}\\
		accuracy    &0.5605	    &0.5599     &0.5607     &0.5607\\
		\hline
	\end{tabular}
	\label{tab:public-2}
\end{table}
	
Here we still use $\delta=0.05$ for calculating the CF-bound. For the COMPAS data, both the CF-bound and CF-metric show that the FLAP methods using the marginal distribution mapping preprocessing procedure are fairer than other fair learning algorithms. For the adult income data, while the CF-metric supports the FLAP methods using the marginal distribution mapping preprocessing procedure, the CF-bound is in favor of the FLAP-2 algorithm with the orthogonalization procedure. Considering the fact that the CF-bound is very high for all methods, it is likely that the adult income data does not satisfy the non-sensitive conditions and none of the methods achieves CF. The accuracy of all fair learning algorithms is comparable to the ML method for both datasets.

Tables \ref{tab:delta-adult}, \ref{tab:delta-compas} show the effect of $\delta$ on the CF-bound for the adult income data and the COMPAS data, respectively. Similar to Table 2, higher $\delta$ results in higher bound and the metrics become similar to each other when $\delta$ is close to 1. Our choice of $0.05$ produces results that are mostly consistent with those using lower or higher $\delta$'s.   

\begin{table}[!htbp]
	\centering
	\caption{Comparison of the CF-bound of decision making algorithms on the adult income data using different range parameter $\delta$.}
	\begin{tabular}{c|ccccc}
		\hline
		        &$\delta=0$ &$\delta=.025$  &$\delta=.05$   &$\delta=.1$    &$\delta=1$\\\hline
		ML      &0.9790     &0.9067         &0.9152         &0.9189         &0.9336\\
		FTU     &0.9689     &0.8236         &0.8421         &0.8512         &0.8901\\
		FL      &0.8715     &0.6834         &0.7166         &0.7196         &0.7991\\
		AA      &0.9408     &0.7363         &0.7656         &0.7655         &0.7990\\
		FLAP-1(O)  &0.9303  &0.7055         &0.7357         &0.7230         &\textbf{0.7803}\\
		FLAP-2(O)  &0.8715  &0.6809         &\textbf{0.7151}&\textbf{0.7056}&0.7905\\
		FLAP-1(M)  &0.5366  &0.6728         &0.7721         &0.7853         &0.7932\\
		FLAP-2(M)  &\textbf{0.5206}&\textbf{0.6471}&0.7470  &0.7587         &0.7976\\
		\hline
	\end{tabular}
	\label{tab:delta-adult}
\end{table}

\begin{table}[!htbp]
	\centering
	\caption{Comparison of the CF-bound of decision making algorithms on the COMPAS data using different range parameter $\delta$.}
	\begin{tabular}{c|ccccc}
		\hline
		        &$\delta=0$ &$\delta=.025$  &$\delta=.05$   &$\delta=.1$    &$\delta=1$\\\hline
		ML      &0.5770     &0.6210         &0.6087         &0.6918         &0.6946\\
		FTU     &0.5043     &0.5526         &0.5892         &0.6332         &0.6393\\
		FL      &\textbf{0.4587}&0.4516     &0.4956         &0.5249         &0.5480\\
		AA      &0.4620     &0.4578         &0.4961         &0.5348         &0.5507\\
		FLAP-1(O)  &0.4623  &0.4418         &0.4852         &0.5354         &0.5450\\
		FLAP-2(O)  &\textbf{0.4587}  &0.4421&0.4854         &0.5350         &\textbf{0.5431}\\
		FLAP-1(M)  &0.7709  &\textbf{0.3735}&0.4012         &0.4484         &0.5645\\
		FLAP-2(M)  &0.7763  &0.3839         &\textbf{0.4007}&\textbf{0.4472}&0.5620\\
		\hline
	\end{tabular}
	\label{tab:delta-compas}
\end{table}

\section{Discussion}
We propose two data preprocessing procedures and the FLAP algorithm to make counterfactually fair decisions. The algorithm is general enough so that any learning methods from logistic regression to neural networks can be used, and counterfactual fairness is guaranteed regardless of the learning methods. The orthogonalization procedure is faster and ensures counterfactually fair decisions when the strong non-sensitive condition is met. The marginal distribution mapping procedure is more complex but guarantees fairness under the weaker non-sensitive condition, which is satisfied by most common types of non-sensitive attributes after reparameterization. Even when the non-sensitive attributes contains mixtures of continuous and discrete variables, the FLAP method is still fairer than other methods we considered as shown in our data analysis.

We also prove the equivalence between counterfactual fairness and the conditional independence of decisions and sensitive attributes given the processed non-sensitive attributes under the non-sensitive assumptions. We illustrate that the CDC test is reliable for testing counterfactual fairness when the sample size is small. When the size gets bigger, however, we need a more efficient testing method for the fairness test.

\bibliography{main}

\begin{thebibliography}{}

\bibitem[Ajunwa et~al., 2016]{ajunwa2016hiring}
Ajunwa, I., Scheidegger, C.~E., and Venkatasubramanian, S. (2016).
\newblock Hiring by algorithm: predicting and preventing disparate impact.
\newblock Presented at the Yale Law School Information Society Project
  conference Unlocking the Black Box: The Promise and Limits of Algorithmic
  Accountability in the Professions.

\bibitem[Angwin and Larson, 2016]{angwin2016bias}
Angwin, J. and Larson, J. (2016).
\newblock Bias in criminal risk scores is mathematically inevitable,
  researchers say.
\newblock {\em Propublica}.

\bibitem[Brennan et~al., 2009]{brennan2009evaluating}
Brennan, T., Dieterich, W., and Ehret, B. (2009).
\newblock Evaluating the predictive validity of the compas risk and needs
  assessment system.
\newblock {\em Criminal Justice and Behavior}, 36(1):21--40.

\bibitem[Chouldechova, 2017]{chouldechova2017fair}
Chouldechova, A. (2017).
\newblock Fair prediction with disparate impact: A study of bias in recidivism
  prediction instruments.
\newblock {\em Big data}, 5(2):153--163.

\bibitem[DeDeo, 2014]{dedeo2014wrong}
DeDeo, S. (2014).
\newblock Wrong side of the tracks: Big data and protected categories.
\newblock {\em arXiv preprint arXiv:1412.4643}.

\bibitem[Dwork et~al., 2012]{dwork2012fairness}
Dwork, C., Hardt, M., Pitassi, T., Reingold, O., and Zemel, R. (2012).
\newblock Fairness through awareness.
\newblock In {\em Proceedings of the 3rd innovations in theoretical computer
  science conference}, pages 214--226.

\bibitem[Dwoskin, 2015]{dwoskin2015social}
Dwoskin, E. (2015).
\newblock How social bias creeps into web technology.
\newblock {\em The Wall Street Journal}, 21.

\bibitem[{Executive Office of the President} et~al., 2016]{executive2016big}
{Executive Office of the President}, Munoz, C., Director, D. P.~C., of~Science,
  M. U. C. T. O. S.~O., Policy)), T., for Data~Policy, D. D. C. T.~O.,
  of~Science, C. D. S. P.~O., and Policy)), T. (2016).
\newblock {\em Big data: A report on algorithmic systems, opportunity, and
  civil rights}.
\newblock Executive Office of the President.

\bibitem[Fuster et~al., 2020]{fuster2020predictably}
Fuster, A., Goldsmith-Pinkham, P., Ramadorai, T., and Walther, A. (2020).
\newblock Predictably unequal? the effects of machine learning on credit
  markets.
\newblock {\em The Effects of Machine Learning on Credit Markets (October 1,
  2020)}.

\bibitem[Hardt et~al., 2016]{hardt2016equality}
Hardt, M., Price, E., and Srebro, N. (2016).
\newblock Equality of opportunity in supervised learning.
\newblock In {\em Advances in neural information processing systems}, pages
  3315--3323.

\bibitem[Khademi et~al., 2019]{khademi2019fairness}
Khademi, A., Lee, S., Foley, D., and Honavar, V. (2019).
\newblock Fairness in algorithmic decision making: An excursion through the
  lens of causality.
\newblock In {\em The World Wide Web Conference}, pages 2907--2914.

\bibitem[Kilbertus et~al., 2017]{kilbertus2017avoiding}
Kilbertus, N., Carulla, M.~R., Parascandolo, G., Hardt, M., Janzing, D., and
  Sch{\"o}lkopf, B. (2017).
\newblock Avoiding discrimination through causal reasoning.
\newblock In {\em Advances in Neural Information Processing Systems}, pages
  656--666.

\bibitem[Kusner et~al., 2017]{kusner2017counterfactual}
Kusner, M.~J., Loftus, J., Russell, C., and Silva, R. (2017).
\newblock Counterfactual fairness.
\newblock In {\em Advances in Neural Information Processing Systems}, pages
  4066--4076.

\bibitem[Nabi and Shpitser, 2018]{nabi2018fair}
Nabi, R. and Shpitser, I. (2018).
\newblock Fair inference on outcomes.
\newblock In {\em Proceedings of the AAAI Conference on Artificial
  Intelligence}, volume 2018, page 1931.

\bibitem[{Nature Editorial}, 2016]{nature2016more}
{Nature Editorial} (2016).
\newblock More accountability for big-data algorithms.
\newblock {\em Nature}, 537(7621):449.

\bibitem[Pearl, 2009a]{pearl2009causal}
Pearl, J. (2009a).
\newblock Causal inference in statistics: An overview.
\newblock {\em Statistics surveys}, 3:96--146.

\bibitem[Pearl, 2009b]{pearl2009causality}
Pearl, J. (2009b).
\newblock {\em Causality}.
\newblock Cambridge university press.

\bibitem[Thomas, 2009]{thomas2009consumer}
Thomas, L.~C. (2009).
\newblock {\em Consumer credit models: pricing, profit and portfolios}.
\newblock OUP Oxford.

\bibitem[Wang et~al., 2015]{wang2015conditional}
Wang, X., Pan, W., Hu, W., Tian, Y., and Zhang, H. (2015).
\newblock Conditional distance correlation.
\newblock {\em Journal of the American Statistical Association},
  110(512):1726--1734.

\bibitem[Wang et~al., 2019]{wang2019equal}
Wang, Y., Sridhar, D., and Blei, D.~M. (2019).
\newblock Equal opportunity and affirmative action via counterfactual
  predictions.
\newblock {\em arXiv preprint arXiv:1905.10870}.

\bibitem[Waters and Miikkulainen, 2014]{waters2014grade}
Waters, A. and Miikkulainen, R. (2014).
\newblock Grade: Machine learning support for graduate admissions.
\newblock {\em AI Magazine}, 35(1):64--64.

\bibitem[Wu et~al., 2019a]{wu2019counterfactual}
Wu, Y., Zhang, L., and Wu, X. (2019a).
\newblock Counterfactual fairness: Unidentification, bound and algorithm.
\newblock In {\em Proceedings of the Twenty-Eighth International Joint
  Conference on Artificial Intelligence}, pages 1438--1444.

\bibitem[Wu et~al., 2019b]{wu2019pc}
Wu, Y., Zhang, L., Wu, X., and Tong, H. (2019b).
\newblock Pc-fairness: A unified framework for measuring causality-based
  fairness.
\newblock In {\em Advances in Neural Information Processing Systems}, pages
  3404--3414.

\bibitem[Yeom and Tschantz, 2018]{yeom2018discriminative}
Yeom, S. and Tschantz, M.~C. (2018).
\newblock Discriminative but not discriminatory: A comparison of fairness
  definitions under different worldviews.
\newblock {\em arXiv preprint arXiv:1808.08619}.

\bibitem[Zemel et~al., 2013]{zemel2013learning}
Zemel, R., Wu, Y., Swersky, K., Pitassi, T., and Dwork, C. (2013).
\newblock Learning fair representations.
\newblock In {\em International Conference on Machine Learning}, pages
  325--333.

\bibitem[Zhang and Bareinboim, 2018a]{zhang2018equality}
Zhang, J. and Bareinboim, E. (2018a).
\newblock Equality of opportunity in classification: A causal approach.
\newblock In {\em Advances in Neural Information Processing Systems}, pages
  3671--3681.

\bibitem[Zhang and Bareinboim, 2018b]{zhang2018fairness}
Zhang, J. and Bareinboim, E. (2018b).
\newblock Fairness in decision-making—the causal explanation formula.
\newblock In {\em Proceedings of the AAAI Conference on Artificial
  Intelligence}.

\bibitem[Zhang et~al., 2017]{zhang2017causal}
Zhang, L., Wu, Y., and Wu, X. (2017).
\newblock A causal framework for discovering and removing direct and indirect
  discrimination.
\newblock In {\em Proceedings of the 26th International Joint Conference on
  Artificial Intelligence}, pages 3929--3935.

\end{thebibliography}
\bibliographystyle{apalike}

\end{document}